\documentclass[review]{elsarticle}

\usepackage{lineno,hyperref}
\modulolinenumbers[5]

\usepackage[intlimits]{amsmath}
\usepackage{graphicx,epsfig,url,float}
\usepackage{amsfonts,amssymb,amsthm}
\usepackage{multirow,xcolor,framed}
\usepackage[normalem]{ulem}
\newcommand\hl{\bgroup\markoverwith
  {\textcolor{red!80}{\rule[-.5ex]{2pt}{0.1ex}}}\ULon}

\newtheorem{lemma}{Lemma}

\journal{Journal of \LaTeX\ Templates}

\DeclareGraphicsExtensions{.png,.pdf}

\begin{document}

\begin{frontmatter}

\title{Local Sparse Approximation for Image Restoration with Adaptive Block Size Selection}


\author[mymainaddress]{Sujit~Kumar~Sahoo\corref{mycorrespondingauthor}}
\cortext[mycorrespondingauthor]{Corresponding author}
\ead{sujit@pmail.ntu.edu.sg}

\address[mymainaddress]{Department of Statistics and Applied Probability, Faculty of Science, National University of Singapore, Singapore}

\begin{abstract}
In this paper the problem of image restoration (denoising and inpainting) is approached using sparse approximation of local image blocks. The local image blocks are extracted by sliding square windows over the image. An adaptive block size selection procedure for local sparse approximation is proposed, which affects the global recovery of underlying image. Ideally the adaptive local block selection yields the minimum mean square error (MMSE) in recovered image. This framework gives us a clustered image based on the selected block size, then each cluster is restored separately using sparse approximation. The results obtained using the proposed framework are very much comparable with the recently proposed image restoration techniques.
\end{abstract}

\end{frontmatter}
\linenumbers

\section{Introduction}
\label{REV:sec:denoising}
{The natural images are generally sparse in some transform domain, which makes sparse representation an emerging tool to solve image processing problems.}
\subsection{Inpainting}
\label{REV:sec:inpainting}
Inpainting is a problem of filling up the missing pixels in an image by taking help of the existing pixels. In literatures, inpainting is often referred as disocclusion, which means to remove an obstruction or unmask a masked image. The success of inpainting lies on how well it infers the missing pixels from the observed pixels. It is a simple form of inverse problem, where the task is to estimate an image $X\in{\mathbb{R}^{\sqrt{N}\times \sqrt{N}}}$ from its measurement $Y\in{\mathbb{R}^{\sqrt{N}\times \sqrt{N}}}$ which is obstructed by a binary mask $B\in{{\{0,1\}}^{\sqrt{N}\times \sqrt{N}}}$.
\begin{equation}
\label{REV:eq1}
Y = X\circ{B} : B(i,j) = \left\{ 
  \begin{array}{l l}
    1 & \quad \text{if $(i,j)$ is observed}\\
    0 & \quad \text{if $(i,j)$ is obstructed}\\
  \end{array} \right.
\end{equation}

In literature, the problem of image inpainting has been addressed from different points of view, such as Partial Differential Equation (PDE), variational principle and exemplar region filling. 
An overview of these methods can be found in these recent articles \cite{Bugeau2010,Arias2011}. Apart from theses approaches, use of explicit sparse representation has produced very promising inpainting results \cite{MCA,EM}. Natural images are generally sparse in some transform domain, which makes sparse representation as an emerging tool for solving image processing problems. Inpainting is a fundamental problem in sparse representation which supports the arguments from compressed sensing \cite{CS}, where random sampling is one of the techniques.

In \cite{Sahoo2013a}, sparse representation is used to solve the image inpainting problem by performing local inpainting of small blocks of size $\sqrt{n}\times\sqrt{n}$. Thus, {the missing pixels of these small $\sqrt{n}\times\sqrt{n}$ images are needed to be filled up} individually.  Let's denote $\mathbf{x}\in\mathbb{R}^n$ as a columnized image block, and $\mathbf{b}\in \{0,1\}^n$ be the corresponding binary mask, then the individual corrupt image blocks can be presented as $\mathbf{y} =\mathbf{b}\circ \mathbf{x}$. {It is known} that it is possible to represent $\mathbf{x} = \mathbf{D}\mathbf{s}$ in a suitable dictionary $\mathbf{D} = [\mathbf{d}_1,\mathbf{d}_2,\dots,\mathbf{d}_K]$ as per the standard notations, where $\mathbf{s}\in\mathbb{R}^K$ is sparse (i.e. $\|\mathbf{s}\|_0\ll n$). Hence, {it is assumed} that $\mathbf{y}$ also has the same sparse representation $\mathbf{s}$ in
\[
[(\mathbf{b}{1^T_K})\circ\mathbf{D}] =  [\mathbf{b}\circ\mathbf{d}_1,\mathbf{b}\circ\mathbf{d}_2,\dots,\mathbf{b}\circ\mathbf{d}_K],
\]
where $1_K$ is a vector containing $K$ ones. Therefore, {a dictionary $\mathbf{D}$ is taken}, and estimate the sparse representation $\mathbf{s}$ for each corrupt image block as follows.
\begin{align}
\label{APP:Ieq1}
\hat{\mathbf{s}} &= \arg\min_{\mathbf{s}}{\|\mathbf{s}\|_0} \ \text{ such that }\ {\|\mathbf{y}-[(\mathbf{b}{1^T_K})\circ\mathbf{D}]\mathbf{s}\|_2^2 \leq \epsilon^2},
\end{align} 
where $\epsilon$ is the allowed representation error. After obtaining $\hat{\mathbf{s}}$, {the image block is restored} as $\hat{\mathbf{x}} = \mathbf{D}\hat{\mathbf{s}}$.

\subsection{Denoising}
\label{REV:sec:denoising}
Growth of semiconductor technologies has made the sensor arrays overwhelmingly dense, which makes the sensors more prone to noise. Hence denoising still remains an important research problem in image processing. Denoising is a form of challenging inverse problem, where the task is to estimate the signal $X$ from its measurement $Y$ which is corrupted by additive noise $V$,
\begin{equation}
\label{REV:eq1}
Y = X + V.
\end{equation}
Note that the noise $V$ is commonly modeled as Additive White Gaussian Noise (AWGN). 

In literature, the problem of image denoising has been addressed from different points of view such as statistical modeling, spatial adaptive filtering, and transfer domain thresholding \cite{ImageDenoisingReview}. In recent years image denoising using sparse representation has been proposed. The well known shrinkage algorithm by D. L. Donoho and L. M. Johnstone \cite{donoho95} is one example of such approach. In \cite{Elad2006}, M. Elad and M. Aharon has explicitly used sparsity as a prior for image denoising. In \cite{KLLD}, P. Chatterjee and P. Milanfar have clustered an image into $K$ clusters to enhance the sparse representation via locally learned dictionaries. 

The key idea in \cite{Elad2006} is to obtain a global denoising of the image by denoising overlapped local image blocks. Let's define $\mathbf{R}_{ij}$ as an $n\times N$ matrix that extracts a $\sqrt{n}\times\sqrt{n}$ block $\mathbf{x}_{ij}$ from the columnized image $X$ starting from its 2D coordinate $(i,j)$ \footnote{Basically, $\mathbf{R}_{ij}$ can be viewed as a matrix, which contains $n$ selected rows of an $N\times N$ identity matrix $\mathbf{I}_N$. Hence it picks $n$ elements from an $N$ dimensional vector.}. By sweeping across the coordinates $(i,j)$ of $X$, {overlapping local blocks can be extracted as $\{\forall_{ij} \mathbf{x}_{ij}=\mathbf{R}_{ij}X\}$}. {It is assumed that} there exists a sparse representation for any columnized image block $\mathbf{x}\in\mathbb{R}^n$ on a suitable dictionary $\mathbf{D}\in\mathbb{R}^{n\times K}$. That is,
\begin{align}
\hat{\mathbf{s}} = \arg\min_{\mathbf{s}}{\|\mathbf{s}\|_0} \ \text{ such that }\ {\|\mathbf{x}-\mathbf{D}\mathbf{s}\|_2^2\leq \epsilon^2}
\end{align}
where $\epsilon$ is the representation error tolerance. After obtaining $\hat{\mathbf{s}}$, {the image block is restored} as $\hat{\mathbf{x}} = \mathbf{D}\hat{\mathbf{s}}$.

\subsection{Motivatioon}
\label{BSS:sec:intro}
In the previous subsections, {the notion of image inpainting and denoising using sparse representation has been introduced, where the global image recovery is carried out through recovery of local image blocks.} The two main reasons behind the use of local image blocks are the following - (i) the smaller blocks take lesser computation time and storage space; (ii) the smaller image blocks contain lesser diversity, hence it is easier to obtain a sparse representation with fewer coefficients. Though, how much smaller the block size will be is left to the user, it has an impact on the recovery performance. This impact is due to a change in image content inside a local block with a change in block size. Thus, it will be better, if we can find a suitable block size at each location that performs the optimal recovery of an image. Nevertheless, the task is challenging, because we don't have the original image to verify the recovery performance. The possibilities of numerous block sizes makes it even more complicated. In this paper, {a framework of block size selection is proposed, which bypasses these challenges. Essentially, possible block sizes are prefixed to a limited number}, instead of dwelling around infinite possibilities. {Next, a block size selection criterion is formulated that uses the corrupt image alone. Some background of block size selection is introduced} in the next section, and in the subsequent sections {both the recovery frameworks (inpainting, denoising) is restated} in conjunction with block size selection. \footnote{The initial phases of this work and it's results has been presented in \cite{Sahoo2011i, Sahoo2011d}. This article provides the collective form of the work with detailed description and more results.} 
\section{Local Block Size Selection}
\label{BSS:sec:blocksize}
In order to simplify the global recovery problem, local recoveries are undertaken as small steps. In general, local block size selection plays an important role in the setup of local to global recovery. In the language of signal processing,  this phenomenon of block size selection is often termed as bandwidth selection for local filtering. A natural question arises, that whether an optimal block size should be selected globally or locally. It is relatively easier to find a block size globally which yields the Minimum Mean Square Error (MMSE). Ideally, the optimal block size for local operation should be selected at each location of the image.  This is because the global mean square error (${MSE} =\frac{1}{N}\sum_{ij}{[X(i,j) -\hat{X}(i,j)]^2}$) is a collective contribution of the local mean square errors $\{\forall_{ij}{MSE}_{ij}=[X(i,j) -\hat{X}(i,j)]^2\}$, where $X$ is the original image of size $\sqrt{N}\times\sqrt{N}$ and $\hat{X}$ is the recovered image. Thus, the optimal block size for a pixel location $(i,j)$ is the one that gives minimum ${MSE}_{ij}$. In the absence of the original image $X$, this task becomes very challenging.

An earlier attempt towards adaptive block size selection can be found in \cite{ICI}, where each pixel is estimated pointwise using Local Polynomial Approximation (LPA). Increasing odd sized square blocks $n = n_1<n_2<n_3< \dots$ were taken centering over each pixel $(i,j)$, and the best estimate is obtained as $\hat{X}^{\hat{n}}(i,j)$. The task is to find $\hat{n} = \arg{ \min_{n}{MSE}_{ij}^n}=\arg{ \min_{n}\left[X(i,j) -\hat{X}^n(i,j)\right]^2}$, where $\hat{X}^n(i,j)$ is the obtained polynomial approximation of the pixel $X(i,j)$ with block size $\sqrt{n}\times\sqrt{n}$. At each pixel $(i,j)$, a confidence interval $\mathcal{D}(n) = [L^n, U^n]$ is obtained for all the block sizes $n= n_1<n_2<n_3< \dots$, 
\begin{align}
L_n = \hat{X}^n(i,j) - \gamma.\text{std}\left(\hat{X}^n(i,j)\right)\nonumber,\\
U_n = \hat{X}^n(i,j) + \gamma.\text{std}\left(\hat{X}^n(i,j)\right)\nonumber,
\end{align}
where $\gamma$ is a fixed constant and $\text{std}\left(\hat{X}^n(i,j)\right)$ is the standard deviation of $\hat{X}^n(i,j)$ over different $n$. In order to find the Intersection of Confidence Intervals (ICI), the intervals $\forall_n \mathcal{D}(n)$ are arranged in the increasing order of local block size $n$. The first block size at which all the intervals intersect is decided as the optimal block $\hat{n}$. It is theoretically proven that ICI will often select the block size with minimum ${MSE}_{ij}^n$. However, the success of ICI is dependent on the accurate estimation of $\hat{X}^n(i,j)$ and its standard deviation $\text{std}\left(\hat{X}^n(i,j)\right)$. In addition, ICI has a drawback that it can only be applied to single pixel recovery framework. Since, {more than one pixel of the estimated local blocks are used in the recovery frameworks, ICI will not help us selecting block size.}
\section{Inpainting using Local Sparse Representation}
\label{BSS:sec:inpaint}
In this problem, an image $X\in\mathbb{R}^{\sqrt{N}\times\sqrt{N}}$ is being occluded by a mask $B\in\{0,1\}^{\sqrt{N}\times\sqrt{N}}$, resulting in $Y=B\circ X$, {where ``$\circ$'' multiplies two matrices element wise.}. The goal is to find $\hat{X}$- the closest possible estimation of $X$. {In article \cite{Sahoo2013a}, $\hat{X}$ has been obtained in a simple manner by estimating each non-overlapping local block, where the motive was only to show the competitiveness of SGK \cite{Sahoo2013} dictionary over $K$-SVD \cite{KSVDd}.} However, {a better inpainting result can obtained} by considering overlapping local blocks. Thus, {a block extraction mechanism is adapted} based on the denoising framework of \cite{Elad2006}.

Here, {blocks of size $\sqrt{n}\times\sqrt{n}$ having a center pixel are explicitly considered,} which means $\sqrt{n}$ is an odd number. {A $n\times N$ matrix, $\mathbf{R}_{ij}^n$ is defined, which} extracts a $\sqrt{n}\times\sqrt{n}$ block $\mathbf{y}_{ij}^n$ from a $\sqrt{N}\times\sqrt{N}$ image $Y$ as $\mathbf{y}_{ij}^n = \mathbf{R}_{ij}^n{Y}$, where the block is centered over the pixel $(i,j)$. {Let's recall} that $Y$, $X$ and $B$ are columnized to $N\times1$ vector for this block extraction operation. Hence, sweeping across the 2D coordinates $(i,j)$ of $Y$, overlapping image blocks can be extracted, i.e. $\forall_{ij}\{\mathbf{y}^n_{ij}=\mathbf{R}_{ij}^n{Y}\} \in \mathbb{R}^n$. {The original image block is denoted} as $\mathbf{x}_{ij}^n$, and the corresponding local mask as $\mathbf{b}_{ij}^n \in \{0,1\}^n$, which makes the corrupt image block $\mathbf{y}_{ij}^n=\mathbf{x}_{ij}^n\circ\mathbf{b}_{ij}^n$.

Let $\mathbf{D}^n\in\mathbb{R}^{n\times K}$ be a known dictionary, where $\mathbf{x}^n_{ij}$ has a representation $\mathbf{x}_{ij}^n = \mathbf{D}^n\mathbf{s}_{ij}^n$, such that $\|\mathbf{s}_{ij}^n\|_0\ll n$. Similar to \cite{Sahoo2013a}, $\mathbf{s}_{ij}^n$ can be estimated as follows,
\[
\hat{\mathbf{s}}_{ij}^n = \arg\min_{\mathbf{s}}{\|\mathbf{s}\|_0} \ \text{ such that }\ {\left\|\mathbf{y}_{ij}^n-\left[\left(\mathbf{b}_{ij}^n{1^T_K}\right)\circ\mathbf{D}^n\right]\mathbf{s}\right\|_2^2 \leq \epsilon^2(n)},
\]
where $\epsilon(n)$ is the representation error tolerance. To have equal error tolerance per pixel irrespective of the block size, $\epsilon(n)=3\sqrt{n}$ {is set for the} experiment, which gives an error tolerance of $3$ gray levels per pixel. {Using the estimated sparse representations, the inpainted local image blocks are obtained as $\left\{\forall_{ij}\ \hat{\mathbf{x}}_{ij}^n = \mathbf{D}^n\hat{\mathbf{s}}_{ij}^n\right\}$.} In spite of equal error tolerance per pixel, the estimation mean square error ($\frac{1}{n}\left\|\mathbf{x}_{ij}^n - \hat{\mathbf{x}}_{ij}^n\right\|_2^2$) varies with block size $n$. It is because at some location, dictionary of some block size may fit better with the available pixels than another block size, which basically depends on the image content in that locality. Hence a MMSE based block size selection becomes essential.
\begin{figure}
\centering\vspace{-0.0cm}
\includegraphics[width=4.1in]{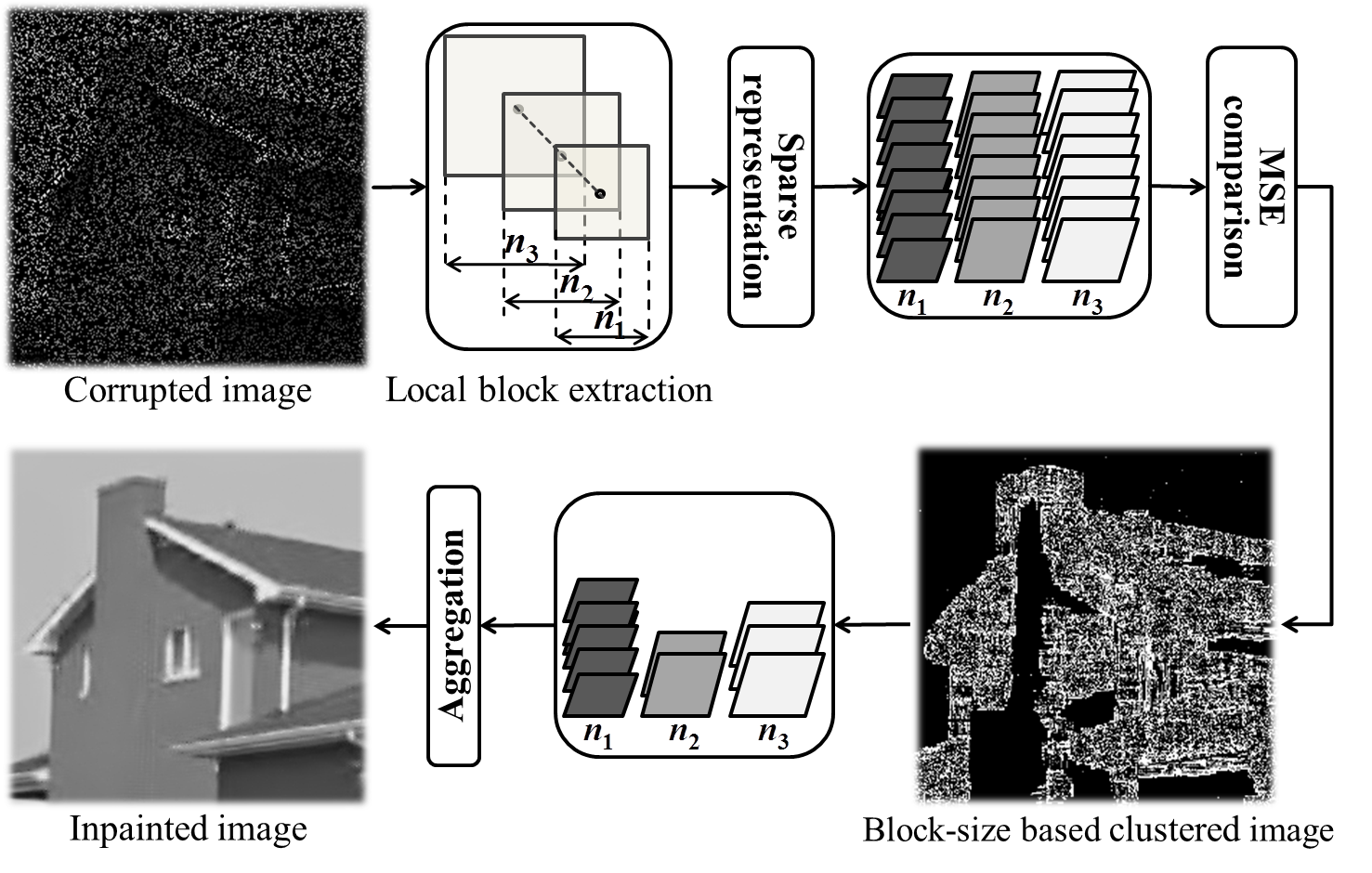}\vspace{-0.15cm}
\caption{Block schematic diagram of the proposed image inpainting framework.}
\label{BSS:Ifig1}
\end{figure}
%
\subsection{Local Block Size Selection for Inpainting}
\label{BSS:imp:blocksize}
The effect of block size is very perceptive in inpainting using local sparse representation. As bigger block sizes capture more details from the image, smaller block sizes are preferred for local sparse representation. However, bigger block sizes are suitable for inpainting as it is hard to follow the trends of the geometrical structures in small block sizes, even in visual perspective. So, there exists a trade-off between the block size and accuracy of fitting. In the absence of the original image, {some measure need to be derived to reach}
\begin{equation}
\label{BSS:Ieq1}
\min_{n} {MSE}_{ij}^n = \min_{n}\frac{1}{n}\left\|\mathbf{x}_{ij}^n - \hat{\mathbf{x}}_{ij}^n\right\|_2^2.
\end{equation}
In order to solve the aforementioned problem an approximation for ${MSE}_{ij}^n$ is carried out. It is done by computing the ${MSE}_{ij}^n$ for the observed pixels only. Thus, it can be written as
\[
\widehat{MSE}_{ij}^n= \frac{1}{{\mathbf{b}_{ij}^n}^T\mathbf{b}_{ij}^n}{\left\|{\mathbf{b}_{ij}^n}\circ\left(\mathbf{x}_{ij}^n - \hat{\mathbf{x}}_{ij}^n\right)\right\|}_2^2 = \frac{1}{{\mathbf{b}_{ij}^n}^T\mathbf{b}_{ij}^n}{\left\|\mathbf{y}_{ij}^n - {\mathbf{b}_{ij}^n}\circ{\hat{\mathbf{x}}_{ij}^n}\right\|}_2^2.
\]
$\widehat{MSE}_{ij}^n$ are computed at each pixel $(i,j)$ for different $n$, and {the block size $\hat{n}=\arg\min_{n}\widehat{MSE}_{ij}^n$ empirically obtained.} Then in a separate image space $W(i,j) = \hat{n}$ is marked, which gives us a clustered image based on the selected block size.
\begin{figure}
\centering\vspace{-0.0cm}
\begin{minipage}[b]{0.327\linewidth}
  \centering
\centerline{\epsfig{figure=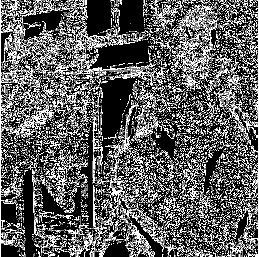,width=4.0cm}}\vspace{-0.18cm}
\centerline{\scriptsize 80\% missing pixel Barbra}
\end{minipage}
\begin{minipage}[b]{0.327\linewidth}
  \centering
\centerline{\epsfig{figure=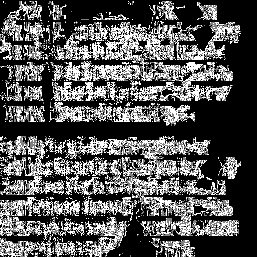,width=4.0cm}}\vspace{-0.19cm}
\centerline{\scriptsize Text printed on Lena} 
\end{minipage}
\begin{minipage}[b]{0.327\linewidth}
  \centering
\centerline{\epsfig{figure=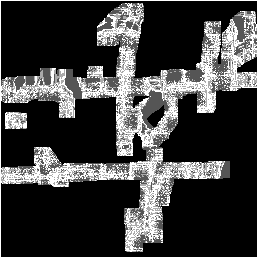,width=4.0cm}}\vspace{-0.18cm}
\centerline{\scriptsize Mascara on Girls image}
\end{minipage}
\caption{Illustration of the block size selection for inpainting.}
\label{BSS:Ifig2}
\end{figure}
\subsection{Implementation Details of Inpainting}
\label{BSS:imp:imp}
The framework is implemented according to the flowchart presented in Figure \ref{BSS:Ifig1}. In practice, the comparison of the sample mean square error will be unfair among the blocks of different size $n = n_1<n_2<n_3< \dots$, because the number of samples are different for each block size. In order to stay unbiased, ${MSE}_{ij}^n$ for each block is computed only over the region covered with the smallest block size $n_1$. The comparison is done in terms of ${\widehat{MSE}_{ij}^n=\frac{1}{{\mathbf{b}_{ij}^{n_1}}^T\mathbf{b}_{ij}^{n_1}}{\left\|\mathbf{R}_{ij}^{n_1}{\mathbf{R}_{ij}^{n}}^T\left(\mathbf{y}_{ij}^n - {\mathbf{b}_{ij}^n}\circ{\hat{\mathbf{x}}_{ij}^n}\right)\right\|}_2^2}$, where $\mathbf{R}_{ij}^{n_1}{\mathbf{R}_{ij}^n}^T$ extracts the common pixels that are covered with block size $n_1$.

Since $\widehat{MSE}_{ij}^n$ only compares the region covered with $n_1$ for any center pixel $(i,j)$, {only those recovered pixels are used, which are covered with $n_1$, that is $\hat{\mathbf{x}}_{ij}^{n_1}=\mathbf{R}_{ij}^{n_1}{\mathbf{R}_{ij}^n}^T\hat{\mathbf{x}}_{ij}^{\hat{n}}$.} Then {the global inpainted image is recovered} from these local inpainted image blocks $\left\{\forall_{ij}\hat{\mathbf{x}}_{ij}^{n_1}\right\}$. Thus, {a MAP estimator is formulated similar to the denoising framework of \cite{Elad2006},}
\[
\hat{X} =\arg\min_{X}\left\{ \lambda\left\|Y-B\circ X\right\|_2^2 +\sum_{ij}{\left\|\mathbf{R}_{ij}^{n_1}X-\hat{\mathbf{x}}_{ij}^{n_1}\right\|_2^2}\right\}.
\]
{Differentiating the right hand side quadratic expression with respect to $X$, the following solution can be obtained.}
\begin{align}
\label{BSS:Ieq2}
-\lambda B\circ\left[Y- B\circ\hat{X}\right]+\sum_{ij}{{\mathbf{R}_{ij}^{n_1}}^T\left[\mathbf{R}_{ij}^{n_1}\hat{X}-\hat{\mathbf{x}}_{ij}^{n_1}\right]} = 0 \nonumber\\
\hat{X} = \left[\lambda\mathrm{diag}(B) + \sum_{ij}{{\mathbf{R}_{ij}^{n_1}}^T \mathbf{R}_{ij}^{n_1}}\right]^{-1}\left[\lambda{Y} + \sum_{ij}{{\mathbf{R}_{ij}^{n_1}}^T\hat{\mathbf{x}}_{ij}^{n_1}}\right]
\end{align}
This expression means that averaging of the inpainted image blocks is to be done, with some relaxation obtained from the corrupt image. Hence $\lambda\propto{1}/{r}$, where $r$ is the fraction of pixels to be inpainted \footnote{{All the experimental results are obtained} keeping $\lambda= {60}/{r}$}. The matrix to invert in the above expression is a diagonal one, hence the calculation of (\ref{BSS:Ieq2}) can be done on a pixel-by-pixel basis after $\{\forall_{ij}\hat{\mathbf{x}}_{ij}^{n_1}\}$ is obtained.

\section{Denoising Using Local Sparse Representation}
\label{BSS:sec:denoise}
Similar to the earlier stated inpainting framework, {square blocks of size $\sqrt{n}\times\sqrt{n}$ with a center pixel are considered}, which means $n$ is an odd number. Sweeping across the coordinate $(i,j)$ of $Y$, {the overlapping local blocks are extracted}, that is $\forall_{ij} \{\mathbf{y}_{ij}^n=\mathbf{R}_{ij}^n{Y}\} \in \mathbb{R}^n$. {The original image block is denoted} as $\mathbf{x}_{ij}^n$, and the noise as $\mathbf{v}_{ij}^n \in \mathbf{N}^n\left(0,\sigma^2\right)$, making the noisy image block $\mathbf{y}_{ij}^n=\mathbf{x}_{ij}^n+\mathbf{v}_{ij}^n$. 
\begin{figure}
\centering\vspace{-0.0cm}
\includegraphics[width=4.1in]{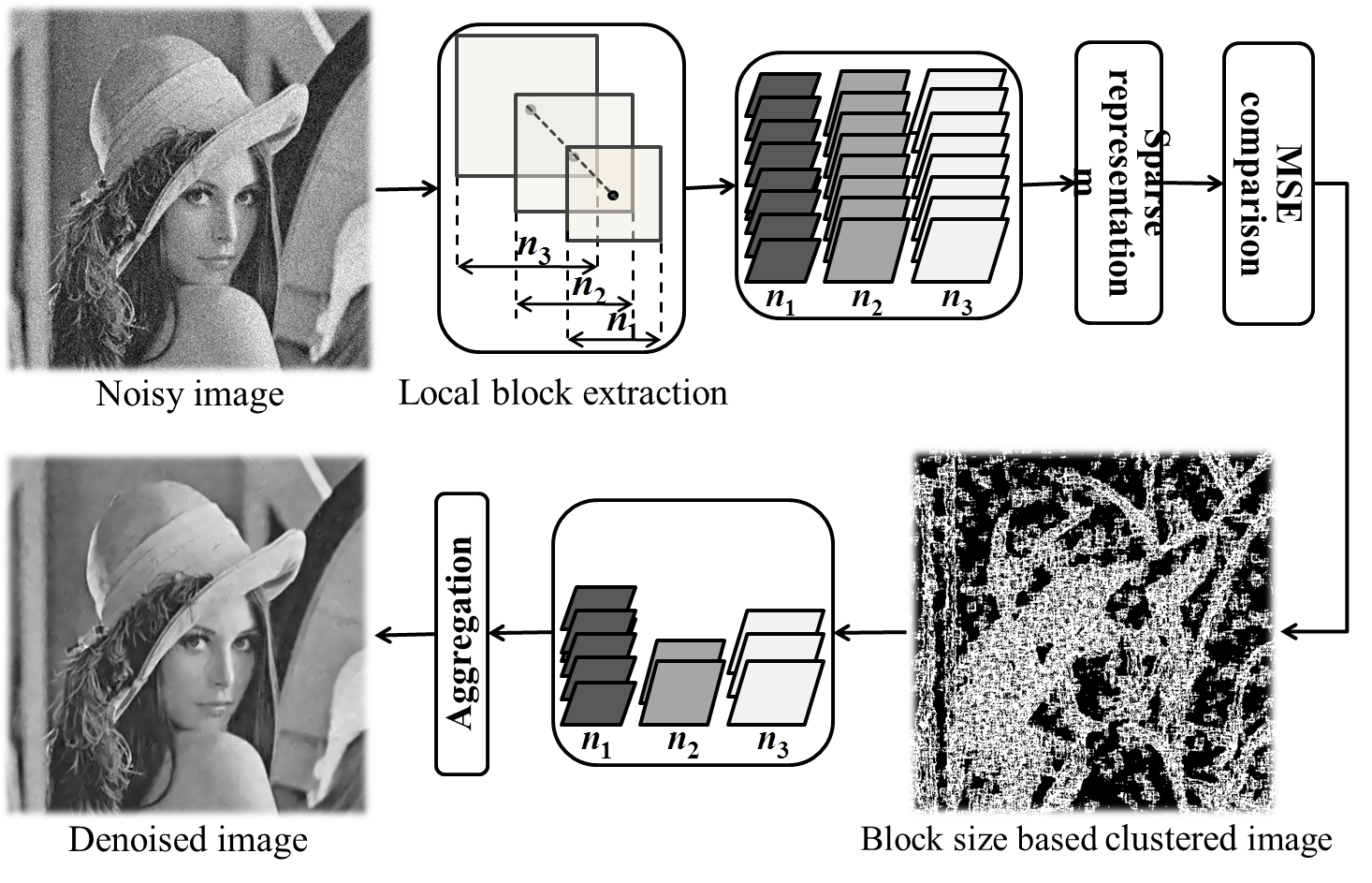}\vspace{-0.15cm}
\caption{Block schematic diagram of the proposed image denoising framework.}
\label{BSS:Dfig1}
\end{figure}

Let $\mathbf{D}^n$ be a known dictionary, where $\mathbf{x}^n_{ij}$ has a representation  $\mathbf{x}_{ij}^n = \mathbf{D}^n\mathbf{s}_{ij}^n$, and $\mathbf{s}_{ij}^n$ is sparse. Since the additive random noise will not be sparse in any dictionary, $\mathbf{s}_{ij}^n$ is estimated as 
\begin{align}
\label{BSS:Deq1}
\hat{\mathbf{s}}_{ij}^n = \arg{ \min_{\mathbf{s}}{\| \mathbf{s}\|}_0 } \;  \mbox{i.e.} \; {\|\mathbf{y}_{ij}^n - \mathbf{D}^n\mathbf{s}\|}_2^2 \leq \epsilon^2(n), 
\end{align}
where $\epsilon(n) \geq \|\mathbf{v}_{ij}^n\|_2$. According to multidimensional Gaussian distribution, if $\mathbf{v}_{ij}^n$ is an $n$ dimensional Gaussian vector, ${\| \mathbf{v}_{ij}^n \|}_2^2$ is distributed by generalized Rayleigh law,
\begin{align} 
\label{BSS:Deq2}
\mathbf{Pr}\left({\left\|\mathbf{v}_{ij}^n\right\|}_2^2 \leq n(1+\varepsilon) \sigma^2\right) = \frac{1}{\Gamma( \frac{n}{2})}\int_{z=0}^{\frac{n(1+\varepsilon)}{2}} z^{ \frac{n}{2} - 1} e^{-z} dz.
\end{align}
By taking $\epsilon^2(n)=n(1+\varepsilon)\sigma^2$, for an appropriately bigger value of $\varepsilon$, it guarantees the sparse representation to be out of the noise radius with high probability. 

Thus, by using the estimated sparse representations, {the denoised local image blocks can be obtained as $\left\{\forall_{ij}\ \hat{\mathbf{x}}_{ij}^n = \mathbf{D}^n\hat{\mathbf{s}}_{ij}^n\right\}$}. Since the increase in block size causes decrease in the correlation between signal and noise, $\varepsilon$ is reduced with increase in $n$ to maintain an equal probability of denoising irrespective of block sizes. In spite of that the mean square error ($\frac{1}{n}\left\|\mathbf{x}_{ij}^n - \hat{\mathbf{x}}_{ij}^n\right\|_2^2$) varies with block size $n$. This is because an equal probability of the estimation being away from the noise radius does not  imply equal closeness to the signal. As the dictionary of some block size matches better with the signal compared to the other, a minimum mean square error (MMSE) based block size selection becomes essential.
\begin{figure}
\centering\vspace{-0.5cm}
\begin{minipage}[b]{0.327\linewidth}
  \centering
\centerline{\epsfig{figure=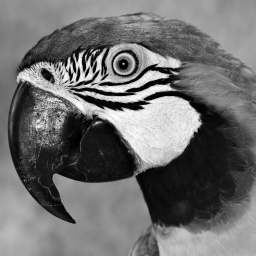,width=4.0cm}}\vspace{-0.3cm}
\centerline{\footnotesize Parrot }
\end{minipage}
\begin{minipage}[b]{0.327\linewidth}
  \centering
\centerline{\epsfig{figure=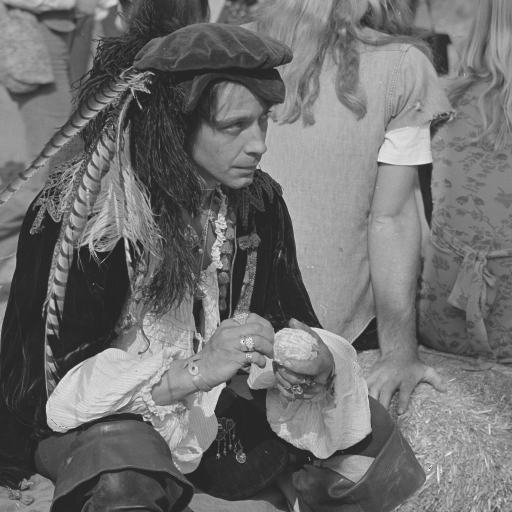,width=4.0cm}}\vspace{-0.3cm}
\centerline{\footnotesize Man }
\end{minipage}
\begin{minipage}[b]{0.327\linewidth}
  \centering
\centerline{\epsfig{figure=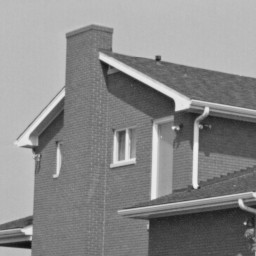,width=4.0cm}}\vspace{-0.3cm}
\centerline{\footnotesize House }
\end{minipage}

\begin{minipage}[b]{0.327\linewidth}
  \centering
\centerline{\epsfig{figure=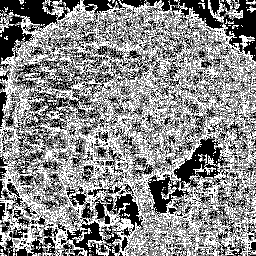,width=4.0cm}}\vspace{-0.3cm}
\centerline{\footnotesize $\sigma=5$}
\end{minipage}
\begin{minipage}[b]{0.327\linewidth}
  \centering
\centerline{\epsfig{figure=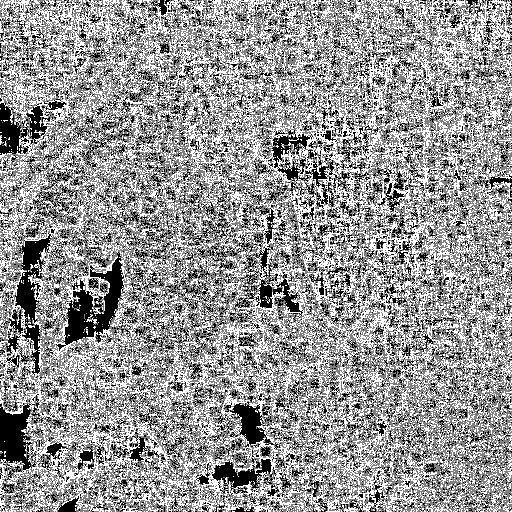,width=4.0cm}}\vspace{-0.3cm}
\centerline{\footnotesize $\sigma=5$}
\end{minipage}
\begin{minipage}[b]{0.327\linewidth}
  \centering
\centerline{\epsfig{figure=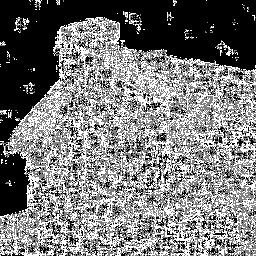,width=4.0cm}}\vspace{-0.3cm}
\centerline{\footnotesize $\sigma=5$}
\end{minipage}

\begin{minipage}[b]{0.327\linewidth}
  \centering
\centerline{\epsfig{figure=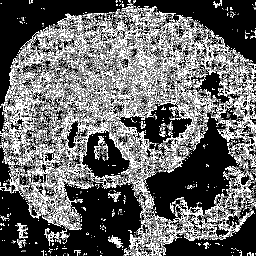,width=4.0cm}}\vspace{-0.3cm}
\centerline{\footnotesize $\sigma=15$}
\end{minipage}
\begin{minipage}[b]{0.327\linewidth}
  \centering
\centerline{\epsfig{figure=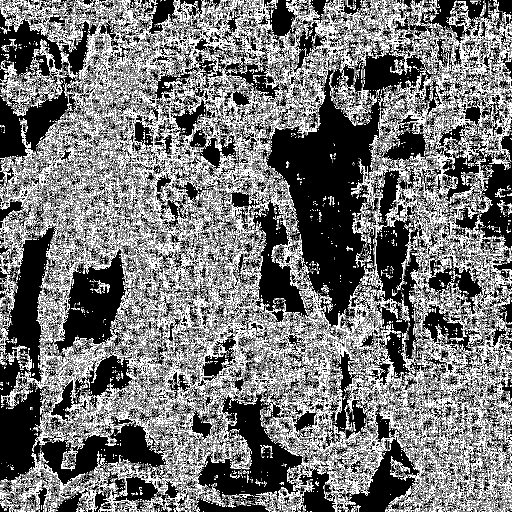,width=4.0cm}}\vspace{-0.3cm}
\centerline{\footnotesize $\sigma=15$}
\end{minipage}
\begin{minipage}[b]{0.327\linewidth}
  \centering
\centerline{\epsfig{figure=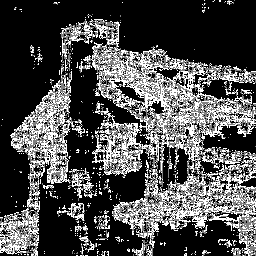,width=4.0cm}}\vspace{-0.3cm}
\centerline{\footnotesize $\sigma=15$}
\end{minipage}

\begin{minipage}[b]{0.327\linewidth}
  \centering
\centerline{\epsfig{figure=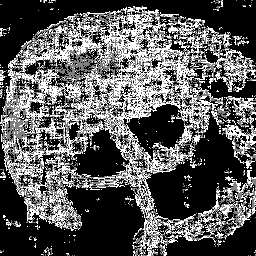,width=4.0cm}}\vspace{-0.3cm}
\centerline{\footnotesize $\sigma=25$}
\end{minipage}
\begin{minipage}[b]{0.327\linewidth}
  \centering
\centerline{\epsfig{figure=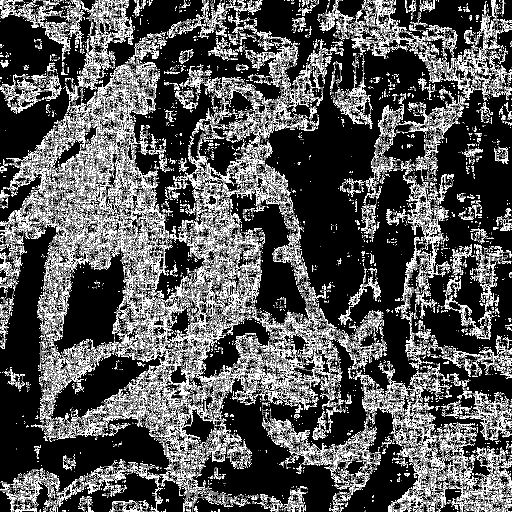,width=4.0cm}}\vspace{-0.3cm}
\centerline{\footnotesize $\sigma=25$}
\end{minipage}
\begin{minipage}[b]{0.327\linewidth}
  \centering
\centerline{\epsfig{figure=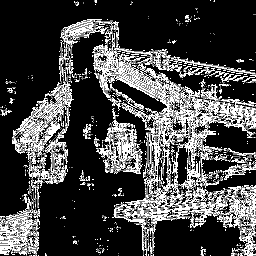,width=4.0cm}}\vspace{-0.3cm}
\centerline{\footnotesize $\sigma=25$}
\end{minipage}\vspace{-0.5cm}
\caption{Illustration of clustering based on block size selection for AWGN of various $\sigma$.}
\label{BSS:Dfig2}
\end{figure}
%
\subsection{Local Block Size Selection for Denoising}
\label{BSS:den:blocksize}
The effect of block size is also very intuitive in denoising using sparse representation: bigger block sizes capture more details from the image, giving rise to more nonzero coefficients. Hence smaller block sizes are preferred for local sparse representation. In contrast, it is hard to distinguish between signal and noise in small sized blocks even in visual perspective,  hence bigger block sizes are suitable for denoising. Thus, there exists a trade-off between the block size and accuracy of fitting. In the absence of the noise free image, {some measure need to be derived to reach}
\begin{equation}
\label{BSS:Deq2}
\min_{n} {MSE}_{ij}^n = \min_{n}\frac{1}{n}\left\|\mathbf{x}_{ij}^n - \hat{\mathbf{x}}_{ij}^n\right\|_2^2.
\end{equation}
In order to solve the aforementioned problem, an approximation for $\min_{n} {MSE}_{ij}^n$ is carried out. It is known that the original image block $\mathbf{x}_{ij}^n=\mathbf{y}_{ij}^n-\mathbf{v}_{ij}^n$, hence after taking the expectation for the noise, it can be written that
\begin{align}
{MSE}_{ij}^n &= \frac{1}{n}\mathbf{E}\left[{\left\|\left(\mathbf{y}_{ij}^n - \hat{\mathbf{x}}_{ij}^n\right) -\mathbf{v}_{ij}^n\right\|}_2^2\right]\nonumber\\
 &=  \frac{1}{n}\mathbf{E}\left[{\left\|\mathbf{y}_{ij}^n - \hat{\mathbf{x}}_{ij}^n\right\|}_2^2\right] - \frac{1}{n}\mathbf{E}\left[\mathbf{v}_{ij}^n\left(\mathbf{y}_{ij}^n - \hat{\mathbf{x}}_{ij}^n\right)^T\right] \nonumber\\
&\qquad\qquad\qquad - \frac{1}{n}\mathbf{E}\left[\left(\mathbf{y}_{ij}^n - \hat{\mathbf{x}}_{ij}^n\right){\mathbf{v}_{ij}^n}^T\right] +  \frac{1}{n}\mathbf{E}\left[{\left\| \mathbf{v}_{ij}^n\right\|}_2^2\right].\nonumber
\end{align}
Heuristically, for a sufficiently large value of $\varepsilon$ in (\ref{BSS:Deq1}) the estimation $\hat{\mathbf{x}}_{ij}^n$ can be kept away from the noise $\mathbf{v}_{ij}^n$. Thus, $\mathbf{E}\left[\mathbf{v}_{ij}^n\left(\mathbf{y}_{ij}^n - \hat{\mathbf{x}}_{ij}^n\right)^T\right] = \mathbf{E}\left[\left(\mathbf{y}_{ij}^n - \hat{\mathbf{x}}_{ij}^n\right){\mathbf{v}_{ij}^n}^T\right]\sim\mathbf{E}\left[{\left\| \mathbf{v}_{ij}^n\right\|}_2^2\right]$, which gives an approximation of ${MSE}_{ij}^n$, 
\[
\widehat{MSE}_{ij}^n= \frac{1}{n}\mathbf{E}\left[{\left\| \mathbf{y}_{ij}^n - \hat{\mathbf{x}}_{ij}^n \right\|}_2^2\right]- \frac{1}{n}\mathbf{E}\left[{\left\|\mathbf{v}_{ij}^n\right\|}_2^2\right].
\]
$\widehat{MSE}_{ij}^n$ are computed at each pixel $(i,j)$ for different $n$, and {the block size $\hat{n}=\arg\min_{n}\widehat{MSE}_{ij}^n$ is obtained} empirically. Then in a separate image space {$W(i,j) = \hat{n}$ is marked}, which gives us a clustered image based on the selected block size.

\subsection{Implementation Details of Denoising }
\label{BSS:den:imp}
The framework is implemented according to the flowchart presented in Figure \ref{BSS:Dfig1}.  In practice, the comparison of the sample mean square error will be unfair among the blocks of different size $n = n_1<n_2<n_3< \dots$, because the number of samples are different for each block size. In order to stay unbiased, ${MSE}_{ij}^n$ for each block is computed only over the region covered with the smallest block size $n_1$. The comparison is done in terms of $\widehat{MSE}_{ij}^n= \frac{1}{n_1}{\left\|\mathbf{R}_{ij}^{n_1}{\mathbf{R}_{ij}^n}^T(\mathbf{y}_{ij}^n - \hat{\mathbf{x}}_{ij}^n)\right\|}_2^2-\frac{1}{n_1}{\left\|\mathbf{v}_{ij}^{n_1}\right\|}_2^2$, where $\mathbf{R}_{ij}^{n_1}{\mathbf{R}_{ij}^n}^T$ extracts the common pixels that are covered with block size $n_1$.

It is also important to ensure that, irrespective of $n$, each estimated $\hat{\mathbf{x}}_{ij}^n$ is noise free with equal probability. {Hence, the following result is established} to maintain equal lower bound probabilities of denoising across $n$.
\begin{lemma}
\label{BSS:lemma1}
For an additive zero mean white Gaussian noise $\mathbf{v}_{ij}^n\in\mathbf{N}[0, \mathbf{I}_n\sigma^2]$, and the observed signal $\mathbf{y}_{ij}^n = \mathbf{D}^n\mathbf{s}_{ij}^n + \mathbf{v}_{ij}^n$, we will have a constant lower-bound for probability $\mathbf{Pr}({\|\mathbf{y}_{ij}^n - \mathbf{D}^n\mathbf{s}_{ij}^n\|}_2^2 < n (1+\varepsilon)\sigma^2)$ over $n$, by taking $\varepsilon = \frac{\varepsilon_0}{\sqrt{n}}$. 
\end{lemma}
\begin{proof}
${\| \mathbf{v}_{ij}^n \|}_2^2$ is a random variable formed out of sum squared of $n$ Gaussian random variables, and $E[{\| \mathbf{v}_{ij}^n\|}_2^2]=n\sigma^2$. Using Chernoff bound \cite{Hoeffding1963}, it can be stated that
\[
\mathbf{Pr}({\| \mathbf{v}_{ij}^n \|}_2^2 \geq n(1+\varepsilon)\sigma^2)  \leq e^{-c_0{\varepsilon^2}n}.
\]
The minimum possible estimation error is ${\|\mathbf{y}_{ij}^n - \mathbf{D}^n\mathbf{s}_{ij}^n\|}_2^2 = {\| \mathbf{v}_{ij}^n \|}_2^2$, and $\mathbf{Pr}( {\| \mathbf{v}_{ij}^n \|}_2^2 < n(1+\varepsilon)\sigma^2) = 1- \mathbf{Pr}({\| \mathbf{v}_{ij}^n \|}_2^2 \geq n(1+\varepsilon)\sigma^2)$. For $\varepsilon =\frac{\epsilon_0}{\sqrt{n}}$, it gives
\[
\mathbf{Pr}({\|\mathbf{y}_{ij}^n - \mathbf{D}^n\mathbf{s}_{ij}^n\|}_2^2 < n(1+\varepsilon)\sigma^2) >1-e^{-c_0{(\frac{\epsilon_0}{\sqrt{n}})^2}n} =1- e^{-c_0{\epsilon_0^2}},
\]
which is a constant lower-bound irrespective of $n$.
\end{proof} 

Similar to the inpainting problem, {the common denoised pixels are extracted} as per the smallest block size $n_1$ after block size is selected for any pixel location $(i,j)$, i.e. $\hat{\mathbf{x}}_{ij}^{n_1}=\mathbf{R}_{ij}^{n_1}{\mathbf{R}_{ij}^n}^T\hat{\mathbf{x}}_{ij}^n$. Then {the overlapping local blocks are average} to recover each pixel of the image, 
\begin{align} 
\label{BSS:Deq3}
\hat{X}=\left[\lambda\mathbf{I}_N + \sum_{ij}{{\mathbf{R}_{ij}^{n_1}}^T \mathbf{R}_{ij}^{n_1}}\right]^{-1}\left[\lambda{Y} +  \sum_{ij}{{\mathbf{R}_{ij}^{n_1}}^T\hat{\mathbf{x}}_{ij}^{n_1}}\right],
\end{align}
which is same as the MAP based local to global recovery in \cite{Elad2006}. 

{It is known} that a better dictionary produces a better denoising result, and that the dictionary training algorithms are capable of performing in presence of noise. Hence, from the noisy image, {trained dictionaries are obtained}, similar to \cite{Elad2006}, and then {the image are denoised} using the block size selection framework presented in Figure \ref{BSS:Dfig1}.
\begin{figure}
\centering\vspace{-0.5cm}
\begin{minipage}[b]{0.327\linewidth}
  \centering
\centerline{\epsfig{figure=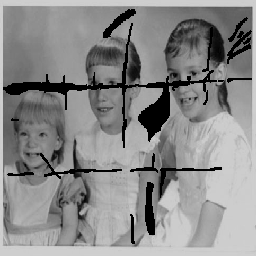,width=4.0cm}}\vspace{-0.3cm}
\centerline{\footnotesize Mascara on Girls}
\end{minipage}
\begin{minipage}[b]{0.327\linewidth}
  \centering
\centerline{\epsfig{figure=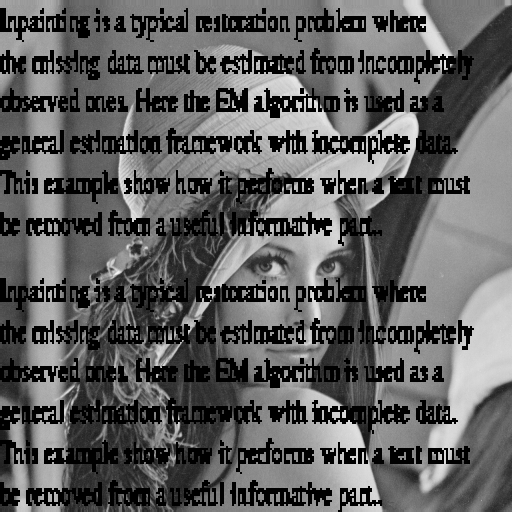,width=4.0cm}}\vspace{-0.3cm}
\centerline{\footnotesize Text on Lena}
\end{minipage}
\begin{minipage}[b]{0.327\linewidth}
  \centering
\centerline{\epsfig{figure=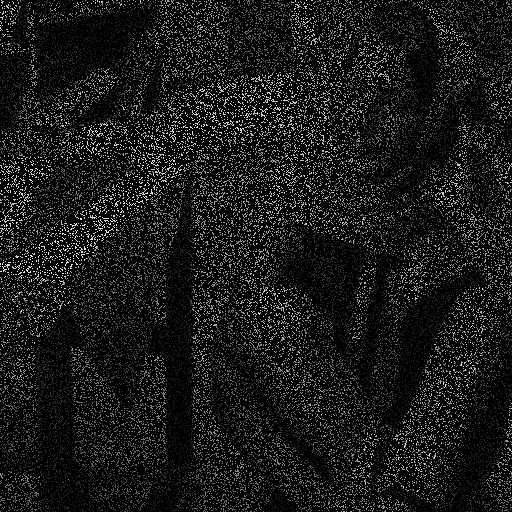,width=4.0cm}}\vspace{-0.3cm}
\centerline{\footnotesize 80\% missing pixel Barbara}
\end{minipage}

\begin{minipage}[b]{0.327\linewidth}
  \centering
\centerline{\epsfig{figure=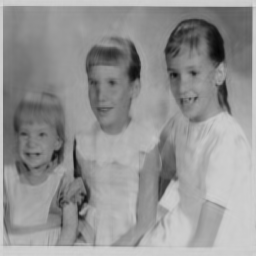,width=4.0cm}}\vspace{-0.3cm}
\centerline{\footnotesize EM \cite{EM}}
\end{minipage}
\begin{minipage}[b]{0.327\linewidth}
  \centering
\centerline{\epsfig{figure=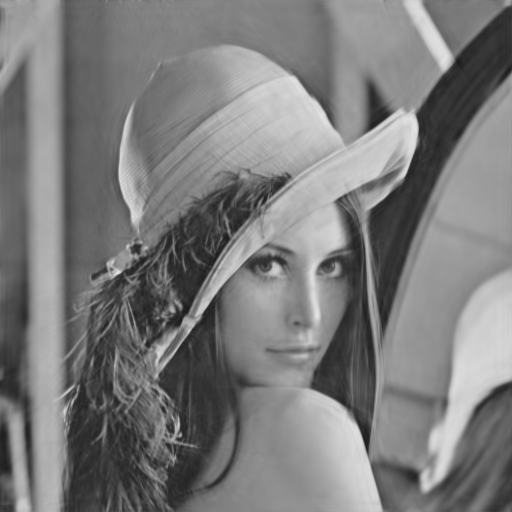,width=4.0cm}}\vspace{-0.3cm}
\centerline{\footnotesize EM \cite{EM} (31.26 dB)}
\end{minipage}
\begin{minipage}[b]{0.327\linewidth}
  \centering
\centerline{\epsfig{figure=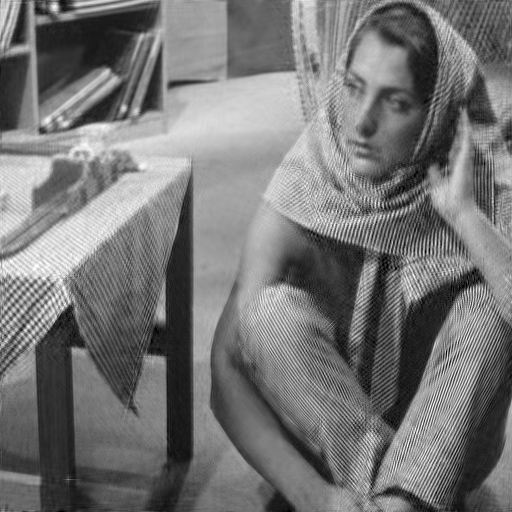,width=4.0cm}}\vspace{-0.3cm}
\centerline{\footnotesize EM \cite{EM} (27.13 dB)}
\end{minipage}

\begin{minipage}[b]{0.327\linewidth}
  \centering
\centerline{\epsfig{figure=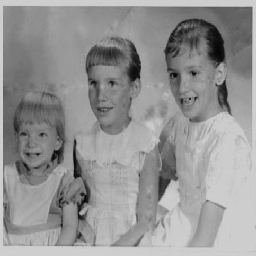,width=4.0cm}}\vspace{-0.3cm}
\centerline{\footnotesize MCA \cite{MCA}}
\end{minipage}
\begin{minipage}[b]{0.327\linewidth}
  \centering
\centerline{\epsfig{figure=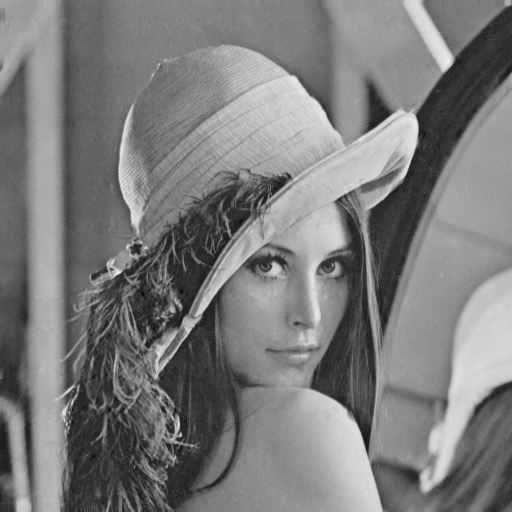,width=4.0cm}}\vspace{-0.3cm}
\centerline{\footnotesize MCA \cite{MCA} (34.18 dB)}
\end{minipage}
\begin{minipage}[b]{0.327\linewidth}
  \centering
\centerline{\epsfig{figure=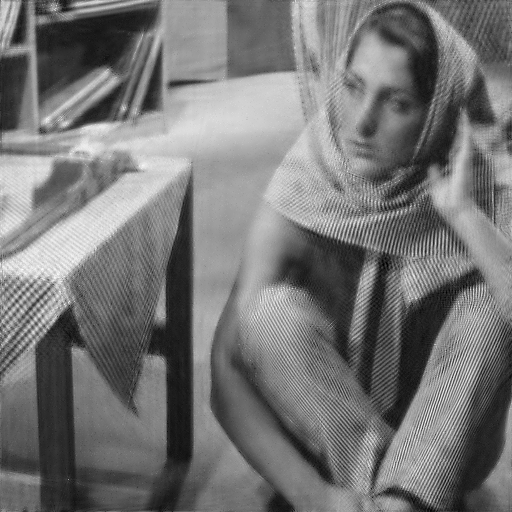,width=4.0cm}}\vspace{-0.3cm}
\centerline{\footnotesize MCA \cite{MCA} (26.62 dB)}
\end{minipage}

\begin{minipage}[b]{0.327\linewidth}
  \centering
\centerline{\epsfig{figure=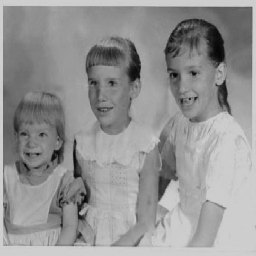,width=4.0cm}}\vspace{-0.3cm}
\centerline{\footnotesize\bf Proposed}
\end{minipage}
\begin{minipage}[b]{0.327\linewidth}
  \centering
\centerline{\epsfig{figure=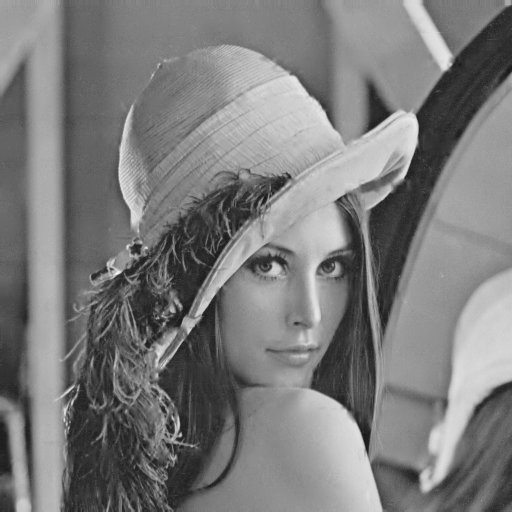,width=4.0cm}}\vspace{-0.3cm}
\centerline{\footnotesize\bf Proposed (34.57 dB)}
\end{minipage}
\begin{minipage}[b]{0.327\linewidth}
  \centering
\centerline{\epsfig{figure=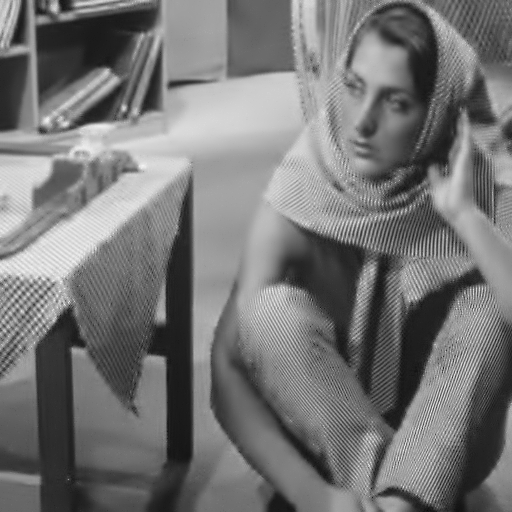,width=4.0cm}}\vspace{-0.3cm}
\centerline{\footnotesize\bf Proposed (27.14 dB)}
\end{minipage}\vspace{-0.5cm}
\caption{Visual comparison of inpainting performance across the methods.}
\label{BSS:Ifig3}
\end{figure}
%
\section{Experimental Results}
\subsection{Inpainting}
\label{BSS:inpaint:exp}
To validate the proposed framework of image inpainting, {it is experimented} on Barbara image with pixels missing at random locations, and the image of girls spoiled by mascara.  {The results are compared} with some of the recently proposed inpainting frameworks ``MCA'' (morphological
component analysis) \cite{MCA} and ``EM'' (expectation maximization)  \cite{EM} . {Local blocks centering over each pixel are extracted} for $256\times256$ images, whereas {local blocks centering over each alternating pixel location of the alternating rows are extracted} for $512\times512$ images. {Overcomplete discrete cosine transform (DCT) dictionary is taken} with $K=4n$ number of atoms for sparse representation. {The error tolerance for sparse representation is set} as $\epsilon(n)=3\sqrt{n}$. {A local block size selection is performed} by taking increasing square block sizes $15\times15,17\times17$ and $19\times19$ as described in section \ref{BSS:imp:blocksize}. Block size based clustered images for different masks $B$ are shown in Figure \ref{BSS:Ifig2} (the gray levels are in increasing order of block size).

After the block sizes have been identified for every location, inpainting is performed for every single local block. Global recovery is done by averaging the overlapped regions as per (\ref{BSS:Ieq2}). The inpainting results for both \cite{MCA} and \cite{EM} are obtained using the MCALab toolbox provided in \cite{MCAlab}. A visual comparison between the proposed framework and the algorithms in \cite{MCA} and \cite{EM} is presented in Figure \ref{BSS:Ifig3}, where mascara is removed from Girls image, text is removed from the Lena, and 80\% of the missing pixels are filled in Barbra image. It can be seen that the images inpainted by the proposed framework are subjectively better in comparison to the rest, since it has more details and fewer artifacts. In terms of quantitative comparison, the proposed framework has also achieved a better Peak Signal to Noise Ratio (PSNR), {which is presented in Table \ref{BSS:Itable1} for the cases of random missing pixels.} 

\begin{table}
\renewcommand{\arraystretch}{1.01}
\small
\caption{\protect{Image inpainting performance comparison in PSNR}}\label{BSS:Itable1}
\centerline{\begin{tabular}{c*{8}{|c}}
\cline{2-8}
& \multicolumn{7}{ c| }{\textbf{Images}}&\\
\hline
\textbf{missing}& Barbra & Lena & Man & Couple & Hill & Boat & Stream& \textbf{Method}\\
\hline\hline
& 32.95 & 34.16 & 29.23  & 31.10  & 31.92  & 31.83  & 25.93 & EM\\
50\% & 31.79 & 32.90 & 29.01 & 30.73 & 31.45 & 31.21 & 26.53 & MCA\\
& \bf34.63 & \bf36.53 & \bf31.09 & \bf32.95 & \bf33.89 & \bf33.27 &\bf 27.29 & Proposed\\
\hline\hline
&17.13 & 29.91 & 24.84 & 26.56 & 27.96 & 26.91 & 22.31 & EM\\
80\% &26.61 & 28.53 & 24.73 & 26.22 & 27.44 & 26.49 & 22.94 & MCA\\
& \bf27.14& \bf29.94 & \bf 25.45 & \bf26.82 & \bf28.47 & \bf26.55 &\bf 23.17 & Proposed\\
\hline
\end{tabular}}
\end{table}
%

\subsection{Denoising}
To validate the proposed framework of image denoising, {it is experimented} on some well known gray scale images corrupted with AWGN ($\sigma = 5,15 \;\mathrm{and}\;25$). {The obtained results are compared} with \cite{Elad2006} ($K$-SVD), and one of its close competitor \cite{KLLD} ($K$-LLD). $K$-LLD is a recently proposed denoising framework, which tried to exceed $K$-SVD's denoising performance by clustering the extracted local image blocks, and by performing sparse representation on each cluster through locally learned dictionaries \footnote{The PCA frame derived from the image blocks of each cluster is defined as the locally learned dictionary. Please note that, number of clusters $K$ of \cite{KLLD} is not the same as number of atoms in the dictionary of {the} proposed framework, it is just a coincidence.}. 
\begin{figure}
\centering \vspace{-5.0mm}
\begin{minipage}[b]{0.327\linewidth}
  \centering
\centerline{\epsfig{figure=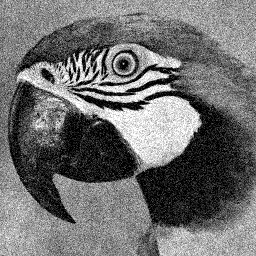,width=4.0cm}}\vspace{-2.8mm}
\centerline{\footnotesize Noisy Parrot}
\end{minipage}
\begin{minipage}[b]{0.327\linewidth}
  \centering
\centerline{\epsfig{figure=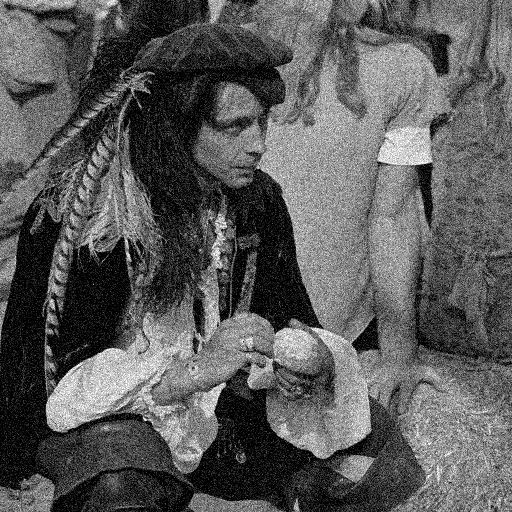,width=4.0cm}}\vspace{-2.8mm}
\centerline{\footnotesize Noisy Man}
\end{minipage}
\begin{minipage}[b]{0.327\linewidth}
  \centering
\centerline{\epsfig{figure=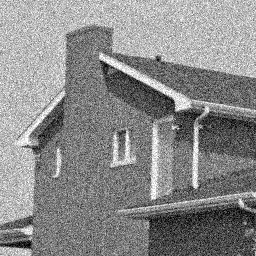,width=4.0cm}}\vspace{-2.8mm}
\centerline{\footnotesize Noisy House}
\end{minipage}

\begin{minipage}[b]{0.327\linewidth}
  \centering
\centerline{\epsfig{figure=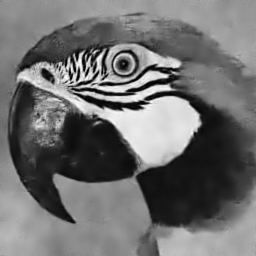,width=4.0cm}}\vspace{-2.8mm}
\centerline{\footnotesize $K$-SVD\cite{Elad2006} (28.43 dB)}
\end{minipage}
\begin{minipage}[b]{0.327\linewidth}
  \centering
\centerline{\epsfig{figure=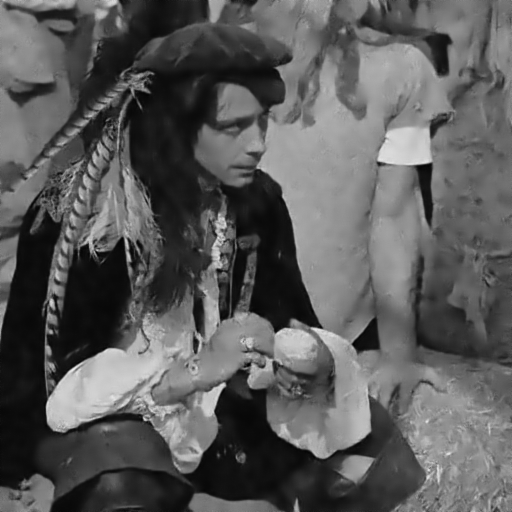,width=4.0cm}}\vspace{-2.8mm}
\centerline{\footnotesize $K$-SVD\cite{Elad2006} (28.11 dB)}
\end{minipage}
\begin{minipage}[b]{0.327\linewidth}
  \centering
\centerline{\epsfig{figure=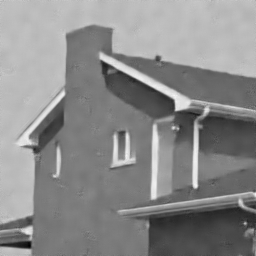,width=4.0cm}}\vspace{-2.8mm}
\centerline{\footnotesize $K$-SVD\cite{Elad2006} (32.10 dB)}
\end{minipage}

\begin{minipage}[b]{0.327\linewidth}
  \centering
\centerline{\epsfig{figure=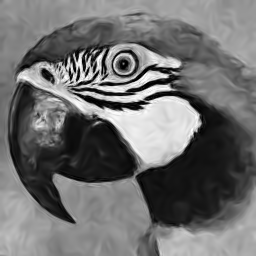,width=4.0cm}}\vspace{-2.8mm}
\centerline{\footnotesize $K$-LLD\cite{KLLD} (27.89 dB)}
\end{minipage}
\begin{minipage}[b]{0.327\linewidth}
  \centering
\centerline{\epsfig{figure=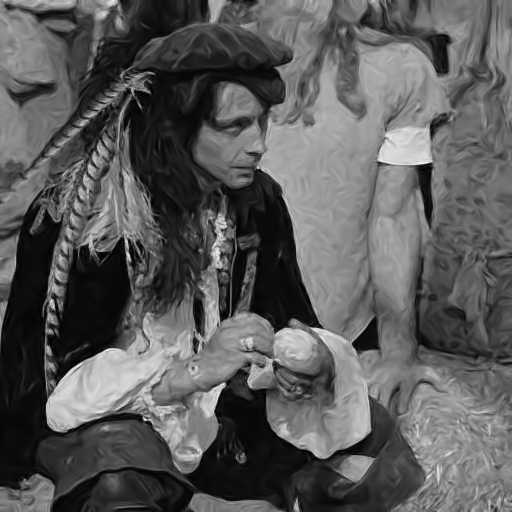,width=4.0cm}}\vspace{-2.8mm}
\centerline{\footnotesize $K$-LLD\cite{KLLD} (28.26 dB)}
\end{minipage}
\begin{minipage}[b]{0.327\linewidth}
  \centering
\centerline{\epsfig{figure=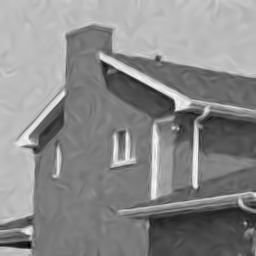,width=4.0cm}}\vspace{-2.8mm}
\centerline{\footnotesize $K$-LLD\cite{KLLD} (30.67 dB)}
\end{minipage}

\begin{minipage}[b]{0.327\linewidth}
  \centering
\centerline{\epsfig{figure=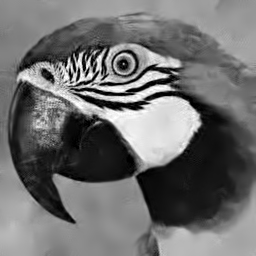,width=4.0cm}}\vspace{-2.8mm}
\centerline{\footnotesize\bf Proposed (28.48 dB)}
\end{minipage}
\begin{minipage}[b]{0.327\linewidth}
  \centering
\centerline{\epsfig{figure=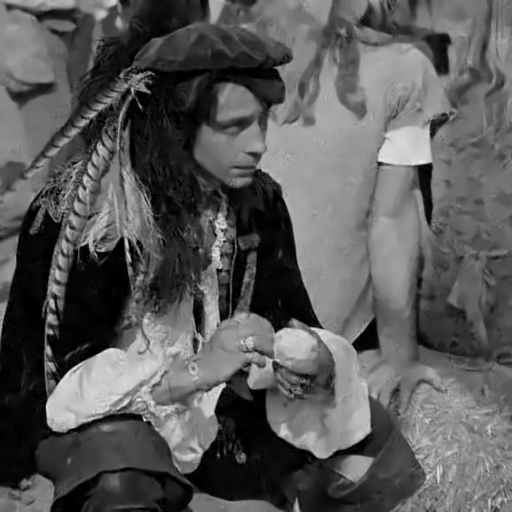,width=4.0cm}}\vspace{-2.8mm}
\centerline{\footnotesize\bf Proposed (28.37 dB)}
\end{minipage}
\begin{minipage}[b]{0.327\linewidth}
  \centering
\centerline{\epsfig{figure=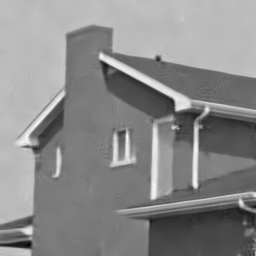,width=4.0cm}}\vspace{-2.8mm}
\centerline{\footnotesize\bf Proposed (32.51 dB)}
\end{minipage}\vspace{-0.5cm}
\caption{Visual comparison of the denoising performances for AWGN $(\sigma=25)$.}
\label{BSS:Dfig3}
\end{figure}

In {the} experimental set up, {local blocks centering over each pixel are extracted} for $256\times256$ images, whereas {local blocks centering over each alternating pixel location of the alternating rows are extracted} for $512\times512$ images. {The number of atoms are kept} as $K=4n$ for each block size $n$. For each block size, to get more than 96\% probability of denoising as per (\ref{BSS:Deq2}), {the value of $\varepsilon =2.68$ is kept} in accordance with Lemma \ref{BSS:lemma1}. {Increasing square blocks of size $11\times11$, $13\times13$ and $15\times15$ are taken}, and selected the local block size as described in section \ref{BSS:den:blocksize}. The selected block size based clustered images are shown in Figure \ref{BSS:Dfig2} (the gray levels are in increasing order of block size). It can be seen clearly that there exists a tradeoff between the noise level and local block size used for sparse representation. When the noise level goes up, a total shift of the clusters from smooth region to texture like region is observed.
\begin{figure}\vspace{-0.25cm}
\centering\vspace{-0.0cm}
\begin{minipage}[b]{0.192\linewidth}
  \centering
\centerline{\epsfig{figure=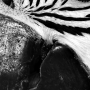,width=2.4cm}}
\vspace{0.8mm}
\end{minipage}
\begin{minipage}[b]{0.192\linewidth}
  \centering
\centerline{\epsfig{figure=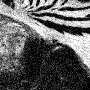,width=2.4cm}}
\vspace{0.8mm}
\end{minipage}
\begin{minipage}[b]{0.192\linewidth}
  \centering
\centerline{\epsfig{figure=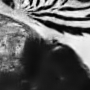,width=2.4cm}}
\vspace{0.8mm}
\end{minipage}
\begin{minipage}[b]{0.192\linewidth}
  \centering
\centerline{\epsfig{figure=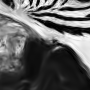,width=2.4cm}}
\vspace{0.8mm}
\end{minipage}
\begin{minipage}[b]{0.192\linewidth}
  \centering
\centerline{\epsfig{figure=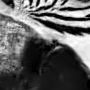,width=2.4cm}}
\vspace{0.8mm}
\end{minipage}
\begin{minipage}[b]{0.192\linewidth}
  \centering
\centerline{\epsfig{figure=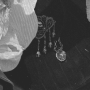,width=2.4cm}}
\vspace{0.8mm}
\end{minipage}
\begin{minipage}[b]{0.192\linewidth}
  \centering
\centerline{\epsfig{figure=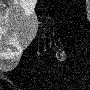,width=2.4cm}}
\vspace{0.8mm}
\end{minipage}
\begin{minipage}[b]{0.192\linewidth}
  \centering
\centerline{\epsfig{figure=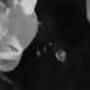,width=2.4cm}}
\vspace{0.8mm}
\end{minipage}
\begin{minipage}[b]{0.192\linewidth}
  \centering
\centerline{\epsfig{figure=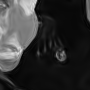,width=2.4cm}}
\vspace{0.8mm}
\end{minipage}
\begin{minipage}[b]{0.192\linewidth}
  \centering
\centerline{\epsfig{figure=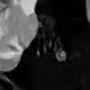,width=2.4cm}}
\vspace{0.8mm}
\end{minipage}
\begin{minipage}[b]{0.192\linewidth}
  \centering
\centerline{\epsfig{figure=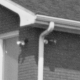,width=2.4cm}}\vspace{-0.18cm}
\centerline{\footnotesize Original}
\end{minipage}
\begin{minipage}[b]{0.192\linewidth}
  \centering
\centerline{\epsfig{figure=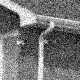,width=2.4cm}}\vspace{-0.18cm}
\centerline{\footnotesize Corrupt}
\end{minipage}
\begin{minipage}[b]{0.192\linewidth}
  \centering
\centerline{\epsfig{figure=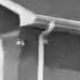,width=2.4cm}}\vspace{-0.18cm}
\centerline{\footnotesize $K$-SVD\cite{Elad2006}}
\end{minipage}
\begin{minipage}[b]{0.192\linewidth}
  \centering
\centerline{\epsfig{figure=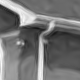,width=2.4cm}}\vspace{-0.18cm}
\centerline{\footnotesize $K$-LLD\cite{KLLD}}
\end{minipage}
\begin{minipage}[b]{0.192\linewidth}
  \centering
\centerline{\epsfig{figure=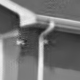,width=2.4cm}}\vspace{-0.18cm}
\centerline{\footnotesize\bf Proposed}
\end{minipage}
\caption{Visual inspection at irregularities}\vspace{-0.1cm}
\label{BSS:Dfig4}
\end{figure}

For each block size, {the trained dictionaries are obtained} from a corrupt image using SGK \cite{Sahoo2013}, in the similar manner as it's done in \cite{Elad2006}. However, number of SGK iterations used are different for different block sizes. Since \cite{Elad2006} has used 10 $K$-SVD iterations for $8\times8$ blocks, {$\lceil10\frac{n}{64}\rceil$ SGK iterations are used} for $\sqrt{n}\times\sqrt{n}$ blocks. After obtaining the trained dictionaries, {the best block size for each location is decided}. Then, {the image is recovered} by averaging the overlapped regions as per (\ref{BSS:Deq3}), by taking $\lambda = 30/\sigma$. 
\begin{table}[tb]\vspace{-0.2cm}
\renewcommand{\arraystretch}{1.01}
\small
\caption{Image denoising performance comparison in PSNR}\label{BSS:Dtable1}
\centerline{\begin{tabular}{c*{8}{|c}}
\cline{2-8}
& \multicolumn{7}{ c| }{\textbf{Images}}&\\
\hline
\textbf{$\sigma$}& CamMan & Parrot & Man & Montage & Peppers & Aerial & House & \textbf{Method}\\
\hline\hline
 &\bf37.90 & \bf37.57 & \bf36.78 & \bf40.17 & \bf37.87 & \bf35.57 & 39.45 & $K$-SVD\\
5 & 36.98 & 36.65 & 36.44 & 39.46 & 37.09 & 35.23 & 37.89 & $K$-LLD\\
& 37.66 & 37.42 & 36.77 & 39.96 & 37.72 & 35.33 & \bf39.51 & Proposed\\
\hline\hline
& \bf31.38 & \bf30.98 & 30.57 & 33.77 & 32.21 & \bf28.64 & 34.32 & $K$-SVD\\
15 & 30.78 & 30.76 & \bf30.76 & 33.14 & 31.96 & 28.55 & 33.89 & $K$-LLD\\
& 31.31 & 30.90 & 30.74 & \bf33.78 & \bf32.25 & 28.49 & \bf34.60 & Proposed\\
\hline\hline
& 28.81 & 28.43 & 28.11 & 30.97 & 29.74 & 25.95 & 32.10 & $K$-SVD\\
25 & 27.96 & 27.89 & 28.26 & 29.52 & 28.94 & 25.78 & 30.67 & $K$-LLD\\
& \bf28.96 & \bf28.48 & \bf28.37 & \bf31.21 & \bf29.91 & \bf25.98 & \bf32.51 & Proposed\\
\hline\hline
& 25.66 & 25.35 & 24.99 & 27.12 & 26.16 & 22.44 & 28.03 & $K$-SVD\\
50 & 20.30 & 20.11 & 20.36 & 20.39 & 20.34 & 19.62 & 20.90 & $K$-LLD\\
& \bf25.92 & \bf25.51 & \bf25.24 & \bf27.35 & \bf26.48 & \bf22.85 & \bf28.66 & Proposed\\
\hline
\end{tabular}}\vspace{-0.1cm}
\end{table}

A visual comparison between the proposed framework and the algorithms in \cite{Elad2006,KLLD} is presented in Figure \ref{BSS:Dfig3}, where the images are heavily corrupted by AWGN $\sigma=25$. In comparison to the rest, it can be seen that the proposed denoising framework produces subjectively better results, since it has more details and fewer artifacts. Notably, the edges in the house image, the complex objects in the man image, and the joint between the mandibles of the parrot image are well recovered. In Figure \ref{BSS:Dfig4} a visual comparison is made for the denoising performance on these diverse and irregular objects. It can be seen that the proposed framework is better. In $K$-LLD denoised image irregularities are heavily smoothed, and a curly artifact is spreading all over. Frameworks like $K$-LLD has the potential to recover the images better, by taking advantage of self similarity inside the images. However, they have a clear drawback when the image has diversity and irregular discontinuity, which has been taken care by block size selection in the proposed frame work. 

A quantitative comparison by PSNR is also made, and results are shown in Table \ref{BSS:Dtable1}. {It can be seen that the proposed framework produces a better PSNR compare to the frameworks in \cite{KLLD}. In the case of higher noise level ($\sigma\geq25$), the proposed framework performs better in comparison to both \cite{Elad2006} and \cite{KLLD}.}
\section{\protect{Discussions}}
\label{BSS:sec:con}
In this paper, image inpainting and denoising using local sparse representation are illustrated in a framework of location adaptive block size selection. This framework is motivated by the importance of block size selection in inferring the geometrical structures and details in the images. It starts with clustering the image based on the block size selected at every location that minimizes the local MSE. Subsequently it aggregates the individual local estimations to estimate the final image. The experimental results show their potential in comparison to the state of the art image recovery techniques. While this paper addresses recovery of gray scale images, it can also be extended to color images. {The} present work provides stimulating results with an intuitive platform for further investigation.

In the present framework, the block sizes are prefixed. However, the bounds on the local block size is an interesting topic to explore further. In the present framework of aggregation, all the pixels of the recovered blocks are given equal weight. An improvement may be achieved by deriving an aggregation formula with adaptive weights per pixel for the recovered local block.

\section*{Acknowledgment}
The author would like to acknowledge Prof. Anamitra Makur for the useful discussions. The author was affiliated to Nanyang Technological University, Singapore during this work, and would like to acknowledge acknowledge their support. 

\end{document}